\documentclass{article}

\makeatletter

\def\Headings#1#2{\def\ps@mypagestyle{\let\@mkboth\@gobbletwo%
\def\@oddhead{\hfill {\small\sc #1} \hfill}%
\def\@oddfoot{\hfill \small\rm \thepage \hfill}%
\def\@evenhead{\hfill {\small\sc #2} \hfill}%
\def\@evenfoot{\hfill \small\rm \thepage \hfill}}%
\pagestyle{mypagestyle}}

\renewcommand\footnoterule{\kern-3\p@ \hrule \@width \textwidth \kern 2\p@}

\def\@startsiction#1#2#3#4#5#6{\if@noskipsec \leavevmode \fi
	\par \@tempskipa #4\relax
	\@afterindenttrue
	\ifdim \@tempskipa <\z@ \@tempskipa -\@tempskipa \@afterindentfalse\fi
	\if@nobreak \everypar{}\else
	\addpenalty{\@secpenalty}\addvspace{\@tempskipa}\fi \@ifstar
	{\@ssect{#3}{#4}{#5}{#6}}{\@dblarg{\@sict{#1}{#2}{#3}{#4}{#5}{#6}}}}

\def\@sict#1#2#3#4#5#6[#7]#8{\ifnum #2>\c@secnumdepth
	\def\@svsec{}\else
	\refstepcounter{#1}\edef\@svsec{\csname the#1\endcsname}\fi
	\@tempskipa #5\relax
	\ifdim \@tempskipa>\z@
	\begingroup #6\relax
	\@hangfrom{\hskip #3\relax\@svsec.\hskip 0.1em}
	{\interlinepenalty \@M #8\par}
	\endgroup
	\csname #1mark\endcsname{#7}\addcontentsline
	{toc}{#1}{\ifnum #2>\c@secnumdepth \else
	\protect\numberline{\csname the#1\endcsname}\fi
	#7}\else
	\def\@svsechd{#6\hskip #3\@svsec #8\csname #1mark\endcsname
	{#7}\addcontentsline
	{toc}{#1}{\ifnum #2>\c@secnumdepth \else
	\protect\numberline{\csname the#1\endcsname}\fi
	#7}}\fi
	\@xsect{#5}}

\def\section{\@startsiction{section}{1}{\z@}{-7.6mm}{2.5mm}{\large\bf\raggedright}}
\def\subsection{\@startsection{subsection}{2}{\z@}{-5mm}{2mm}{\normalsize\bf\raggedright}}
\def\subsubsection{\@startsection{subsubsection}{3}{\z@}{-4.6mm}{2mm}{\bf\raggedright}}

\def\paragraph{\@startsiction{paragraph}{4}{\z@}{1.5ex plus 0.5ex minus .2ex}{-1em}{\normalsize\bf}}
\def\subparagraph{\@startsiction{subparagraph}{5}{\z@}{1.5ex plus 0.5ex minus .2ex}{-1em}{\normalsize\bf}}

\makeatother

\usepackage{amsmath,amsthm,amssymb}
\usepackage{enumerate}
\usepackage{geometry}
\usepackage{cite}
\usepackage[colorlinks,linktoc=all,linkcolor=blue]{hyperref}
\usepackage[capitalise,sort]{cleveref}

\Crefname{setting}{Setting}{Settings}

\renewenvironment{abstract}
{\centerline{\bf Abstract}\vspace{0.7ex}%
	\bgroup\leftskip 40pt\rightskip 40pt\small\noindent\ignorespaces}%
{\par\egroup\vskip 0.7ex}

\renewenvironment{equation*}{\begin{equation}}{\end{equation}}
\numberwithin{equation}{section}

\newtheorem{theorem}{Theorem}[section]
\newtheorem{corollary}[theorem]{Corollary}
\newtheorem{lemma}[theorem]{Lemma}
\newtheorem{proposition}[theorem]{Proposition}
\newtheorem{definition}[theorem]{Definition}
\newtheorem{setting}[theorem]{Setting}

\title{\Large{\bf{Non-convergence of stochastic gradient descent \\ in the training of deep neural networks}}\vskip 1mm}

\author{Patrick Cheridito\footnote{Department of Mathematics, ETH Zurich, patrick.cheridito@math.ethz.ch} \qquad Arnulf Jentzen\footnote{Faculty of Mathematics and Computer Science, University of M{\"u}nster, ajentzen@uni-muenster.de} \qquad Florian Rossmannek\footnote{Department of Mathematics, ETH Zurich, florian.rossmannek@math.ethz.ch}}

\date{}

\Headings{Non-convergence of SGD in the training of DNNs}{Cheridito, Jentzen, Rossmannek}

\newcommand{\E}{\mathbb{E}}
\newcommand{\N}{\mathbb{N}}
\newcommand{\R}{\mathbb{R}}

\newcommand{\ssum}[2]{\mathop{\textstyle{\sum}}_{#1}^{#2}}
\newcommand{\sprod}[2]{\mathop{\textstyle{\prod}}_{#1}^{#2}}

\newcommand{\ReLU}{\rho}
\newcommand{\NNel}{\theta}
\newcommand{\depth}{D}
\newcommand{\arch}{a}
\newcommand{\NNset}{A_d}
\newcommand{\NNaff}[3]{\mathcal{A}_{#1,#2}^{#3}}
\newcommand{\woNNparam}{\mathcal{P}}
\newcommand{\NNparam}{\woNNparam(\arch)}
\newcommand{\NNparaM}[1]{\woNNparam(\arch^{#1})}
\newcommand{\NNfct}[2]{\mathcal{R}_{#1}^{#2}}
\newcommand{\clip}{\mathfrak{c}}
\newcommand{\maxwidth}{W}

\newcommand{\Psp}{\Omega}
\newcommand{\Pel}{\omega}
\newcommand{\Psa}{\mathcal{F}}
\newcommand{\Pm}{\mathbb{P}}
\newcommand{\dataX}[3]{X_{#3}^{#1,#2}}
\newcommand{\dataY}[3]{Y_{#3}^{#1,#2}}
\newcommand{\datX}{\dataX{0}{0}{0}}
\newcommand{\datY}{\dataY{0}{0}{0}}
\newcommand{\Bayes}{\mathcal{B}}
\newcommand{\expB}{\mathcal{E}}

\newcommand{\Loss}[3]{\mathcal{L}_{#3}^{#1,#2}}
\newcommand{\LossDat}[1]{\Loss{0}{0}{#1}}
\newcommand{\fctLoss}[3]{L_{#3}^{#1,#2}}
\newcommand{\trueLoss}[1]{\mathfrak{L}^{#1}}
\newcommand{\finalLoss}{\mathcal{V}}
\newcommand{\SGDstepsize}{\gamma}
\newcommand{\SGDstep}[3]{\Theta_{#3}^{#1,#2}}
\newcommand{\SGDvstep}[2]{\Theta_{#2}^{#1}}
\newcommand{\SGDstepM}[4]{\Theta_{#3,#4}^{#1,#2}}
\newcommand{\SGDgrad}[3]{\mathcal{G}_{#3}^{#1,#2}}
\newcommand{\SGDmin}[1]{\mathbf{k}_{#1}}

\newcommand{\InAct}{\mathrm{I}_{\arch}}
\newcommand{\InActL}[1]{\mathfrak{I}_{\arch,#1}}

\newcommand{\brak}[1]{\left(#1\right)}
\newcommand{\abs}[1]{\left|#1\right|}

\begin{document}

\maketitle

\begin{abstract}
	Deep neural networks have successfully been trained in various application areas with stochastic gradient descent.
However, there exists no rigorous mathematical explanation why this works so well.
The training of neural networks with stochastic gradient descent has four different discretization parameters:
(i) the network architecture;
(ii) the amount of training data;
(iii) the number of gradient steps;
and (iv) the number of randomly initialized gradient trajectories.
While it can be shown that the approximation error converges to zero if all four parameters are sent to infinity in the right order, we demonstrate in this paper that stochastic gradient descent fails to converge for ReLU networks if their depth is much larger than their width and the number of random initializations does not increase to infinity fast enough.

\end{abstract}


\section{Introduction}

Deep learning has produced remarkable results in different practical applications such as image classification, speech recognition, machine translation, and game intelligence.
In this paper, we analyze it in the context of a supervised learning task, though it has also successfully been applied in unsupervised learning and reinforcement learning.
Deep learning is usually implemented with a stochastic gradient descent (SGD) method based on training data.
Gradient descent methods have long been known to work, even with good rates, for convex problems; see, e.g., \cite{BotCurNoc2018}.
However, the training of a deep neural network (DNN) is a non-convex problem, and questions about guarantees and convergence rates of SGD in this context are currently among the most important research topics in the mathematical theory of machine learning.

To obtain optimal approximation results with a DNN, several hyper-parameters have to be fine-tuned.
First, the architecture of the network determines what type of functions can be approximated.
To be able to efficiently approximate complicated functions, it needs to be sufficiently wide and deep.
Secondly, the goal is to approximate the target function with respect to the true risk, but the algorithm only has access to the empirical risk. The gap between the two goes to zero as the amount of training data increases to infinity.
Thirdly, the gradient method attempts to minimize the empirical risk, and the chance of finding a good approximate minimum increases with the number of gradient steps.
Finally, since a single gradient trajectory may not yield good results, it is common to run several of them with different random initializations.
\cite{BeckJenKuck2019,JenWelti2020} have shown that general networks converge if their size, the amount of training data, and the number of random initializations are increased to infinity in the correct way, albeit with an extremely slow speed of convergence.
In general, one cannot hope to overcome the slow speed of converge; see \cite{Shamir2019}.
On the other hand, it has been shown that, for the training error, faster convergence can be guaranteed with certain probabilities 
if over-parametrized networks are used; see
\cite{ChizatOyallonBach2019,DuZhaiPocSin2019,
EMaWu2020,ZouCaoZhouGu2020,SolJavLee2019} and the references therein.
A different approach to the convergence problem relies on landscape analysis of the loss surface.
For example, it is known that there are no local minima if the networks are linear; see \cite{BaldiHornik1989,Kawaguchi2016}.
This is no longer true for non-linear networks\footnote{Unless the loss is measured with respect to a finite data set on which the network is heavily overfitted by, e.g., greatly over-parametrizing the last hidden layer; see \cite{SoudryCarmon2016,LivShaSha2014}.}; see \cite{SafranShamir2018}.
But in this case, there are results about the frequency of local minima; see, e.g., \cite{SafranShamir2016,SafranShamir2018,FukuAma2000,
SwiCzaPas2016,ChoHenMatBenLeCun2015,SoudryHoffer2017}.
The initialization method is important for any type of network. But for ReLU networks it plays a special role
due to the particular form of the ReLU activation function; see \cite{HanRol2018,Hanin2018,LuShinSuKarn2019,ShinKarn2020}.

The main contribution of this paper is a demonstration that SGD fails to converge for ReLU networks 
if the number of random initializations does not increase fast enough compared to the size of the network.
To illustrate our findings, we present a special case of our main result, \cref{thrm_main_quant}, in \cref{thrm_intro} below.

We denote by $d \in \N = \{1,2,\dots\}$ the dimension of the input domain of the approximation problem.
The set $\NNset = \bigcup_{\depth \in \N} (\{d\} \times \N^{\depth-1} \times \{1\})$ represents all network architectures with input dimension $d$ and output dimension 1.
In particular, a vector $a = (\arch_0,\dots,\arch_{\depth}) \in \NNset$ describes the depth $\depth$ of a network and the number of neurons $\arch_0, \dots, \arch_D$ in the different layers.
For any such architecture $a$, the quantity $\NNparam = \sum_{j=1}^{\depth} \arch_j(\arch_{j-1}+1)$ counts the number of real parameters, that is, the number of weights and biases of a DNN with architecture $\arch$.
We consider networks with ReLU activation in the hidden layers and a linear read-out map.
That is, the realization function $\NNfct{\arch}{\NNel} \colon \R^d \rightarrow \R$ of a fully connected feedforward DNN with architecture $\arch = (\arch_0,\dots,\arch_{\depth}) \in \NNset$ and weights and biases $\NNel \in \R^{\NNparam}$ is given by
\begin{equation*}
	\NNfct{\arch}{\NNel} = \NNaff{\arch_{\depth}}{\arch_{\depth-1}}{\NNel,\sum_{i=1}^{\depth -1} \arch_i(\arch_{i-1}+1)} \circ \ReLU \circ \NNaff{\arch_{\depth-1}}{\arch_{\depth-2}}{\NNel,\sum_{i=1}^{\depth -2} \arch_i(\arch_{i-1}+1)} \circ \ReLU \circ \dots \circ \NNaff{\arch_2}{\arch_1}{\NNel,\arch_1(\arch_0+1)} \circ \ReLU \circ \NNaff{\arch_1}{\arch_0}{\NNel,0},
\end{equation*}%
where $\NNaff{m}{n}{\NNel,k} \colon \R^n \rightarrow \R^m$ denotes the affine mapping
\begin{equation*}
\label{intro_aff_map}
	(x_1,\dots,x_n) \mapsto
	\begin{pmatrix}
	\NNel_{k+1} & \NNel_{k+2} & \cdots & \NNel_{k+n} \\
	\NNel_{k+n+1} & \NNel_{k+n+2} & \cdots & \NNel_{k+2n} \\
	\vdots & \vdots & \ddots & \vdots \\
	\NNel_{k+(m-1)n+1} & \NNel_{k+(m-1)n+2} & \cdots & \NNel_{k+mn}
	\end{pmatrix}
	\begin{pmatrix}
	x_1 \\
	x_2 \\
	\vdots \\
	x_n
	\end{pmatrix}
	+
	\begin{pmatrix}
	\NNel_{k+mn+1} \\
	\NNel_{k+mn+2} \\
	\vdots \\
	\NNel_{k+mn+m}
	\end{pmatrix}
\end{equation*}%
and $\ReLU \colon \bigcup_{k \in \N} \R^k \rightarrow \bigcup_{k \in \N} \R^k$ is the ReLU function
$(x_1,\dots,x_k) \mapsto (\max\{x_1,0\},\dots,\max\{x_k,0\})$.
In the following description of the SGD algorithm, $n \in \N$ is the index of the trajectory, $t \in \N_0$ represents the index of the step along the trajectory, $m \in \N$ denotes the batch size of the empirical risk, and $\arch \in \NNset$ describes the architecture under consideration.
We assume the training data is given by functions $\dataX{n}{t}{j} \colon \Psp \rightarrow [0,1]^d$ and $\dataY{n}{t}{j} \colon \Psp \rightarrow [0,1]$, $j,n,t \in \N_0$, on a given probability space $(\Psp,\Psa,\Pm)$.
In a typical learning problem, $(\dataX{n}{t}{j},\dataY{n}{t}{j})$, $j,n,t \in \N_0$, are i.i.d. random variables.
But for Theorem \ref{thrm_intro} to hold, it is enough if $(\dataX{0}{0}{j},\dataY{0}{0}{j})$, $j \in \N_0$, are i.i.d.\ random 
variables, whereas $(\dataX{n}{t}{j},\dataY{n}{t}{j}) \colon \Psp \rightarrow [0,1]^{d+1}$ are arbitrary mappings for $(n,t) \neq (0,0)$.
The target function $\expB \colon [0,1]^d \rightarrow [0,1]$ we are trying to learn is the factorized conditional expectation given ($\Pm$-a.s.) by $\expB(\datX) = \E[\datY | \datX]$.
The empirical risk used for training is
\begin{equation}
\label{intro_emp_risk}
\Loss{n}{t}{\arch,m}(\NNel) = \frac{1}{m} \ssum{j=1}{m} \big| \clip \circ \NNfct{\arch}{\NNel}(\dataX{n}{t}{j}) - \dataY{n}{t}{j} \big|^2,
\end{equation}%
where we compose the network realization with the clipping function $\clip(x) = \max\{0,\min\{x,1\}\}$.
This composition inside the risk is equivalent to a non-linear read-out map of the network.
However, it is more convenient for us to view $\clip$ as part of the risk criterion instead of the network.
But this is only a matter of notation.
Observe that \eqref{intro_emp_risk} is a supervised learning task with noise since, in general, the best possible least squares approximation of $\dataY{0}{0}{0}$ with a deterministic function of $\datX$ is $\expB(\datX)$, which is only equal to $Y^{0,0}_0$ in the special case where $Y^{0,0}_0$ is $X^{0,0}_0$-measurable.
We let $\SGDgrad{n}{t}{\arch,m} \colon \R^{\NNparam} \times \Psp \rightarrow \R^{\NNparam}$ be a function that is equal to the gradient of $\Loss{n}{t}{\arch,m}$ where it exists. The trajectories of the SGD algorithm are given by random variables $\SGDstep{n}{t}{\arch,m} \colon \Psp \rightarrow \R^{\NNparam}$ satisfying the defining relation
\begin{equation*}
\SGDstep{n}{t}{\arch,m} = \SGDstep{n}{t-1}{\arch,m} - \SGDstepsize_t \SGDgrad{n}{t}{\arch,m}(\SGDstep{n}{t-1}{\arch,m})
\end{equation*}%
for given step sizes $(\SGDstepsize_t)_{t \in \N} \subseteq \R$.
Now, we are ready to state the following result, which is a consequence of Theorem 6.5 in \cite{JenWelti2020} and \cref{cor_main_qual} below.

\begin{theorem}
\label{thrm_intro}
Assume that the target function $\expB$ is Lipschitz continuous and that $\expB(\datX)$ is not $\Pm$-a.s.-constant.
Suppose that, for all $a \in A_d$ and $m \in \N$, the random initializations $\SGDstep{n}{0}{\arch,m}$, $n \in \N$,
are independent and uniformly distributed on $[-c,c]^{\NNparam}$, where $c \in [2,\infty)$ is larger than the Lipschitz constant of $\expB$.
Let $\SGDmin{\arch,M,N,T} \colon \Psp \rightarrow \N \times \N_0$ be random variables satisfying
\begin{equation}
\label{thrm_intro_disp_argmin}
\SGDmin{\arch,M,N,T}(\Pel) \in \mathrm{argmin}_{(n,t) \in \{1,\dots,N\} \times \{0,\dots,T\}, \,\SGDstep{n}{t}{\arch,M}
(\Pel) \in [-c,c]^{\NNparam}} \LossDat{\arch,M}(\SGDstep{n}{t}{\arch,M}(\Pel),\Pel).
\end{equation}%
Then, one has
\begin{equation}
\label{thrm_intro_disp_1}
\limsup_{\substack{\arch = (\arch_0,\dots,\arch_{\depth}) \in \NNset \\ \min\{\depth,\arch_1,\dots,\arch_{\depth-1}\} \rightarrow \infty}} \limsup_{\substack{M,N \in \N \\ \min\{M,N\} \rightarrow \infty}} \sup_{T \in \N_0} \E\bigg[\min\!\bigg\{ \int_{[0,1]^d} \abs{ \brak{\clip \circ \NNfct{\arch}{\SGDvstep{\SGDmin{\arch,M,N,T}}{\arch,M}}}(x) - \expB(x) }\, \Pm_{\datX}(dx) , 1 \bigg\}\bigg] = 0
\end{equation}%
and
\begin{equation}
\label{thrm_intro_disp_2}
	\inf_{N \in \N} \limsup_{\substack{\arch = (\arch_0,\dots,\arch_{\depth}) \in \NNset \\ \min\{\depth,\arch_1,\dots,\arch_{\depth-1}\} \rightarrow \infty}} \inf_{\substack{M \in \N \\ T \in \N_0}} \E\bigg[\min\!\bigg\{ \int_{[0,1]^d} \abs{\brak{\clip \circ \NNfct{\arch}{\SGDvstep{\SGDmin{\arch,M,N,T}}{\arch,M}}}(x) - \expB(x) }\, \Pm_{\datX}(dx) , 1 \bigg\}\bigg] >0.
\end{equation}%
\end{theorem}

The integrals in \eqref{thrm_intro_disp_1} and \eqref{thrm_intro_disp_2} describe the true risk.
Note that in Theorem \ref{thrm_intro} the random initializations of the different trajectories are assumed to be independent uniformly distributed on the hypercube $[-c,c]^{\NNparam}$, but our main result, Theorem \ref{thrm_main_quant} below, also covers more general cases.
The random variable $\SGDmin{\arch,M,N,T}$ determines the specific trajectory and gradient step among the first $N$ trajectories and $T$ steps which minimize the empirical risk corresponding to batch size $M$.
Note that $\expB(\datX)$ not being a.s.-constant is a weak assumption since it merely means that the learning task is non-trivial.
Moreover, the stronger condition that $\expB$ must be Lipschitz continuous is made only to ensure the validity of the positive result (\ref{thrm_intro_disp_1}), whereas our new contribution (\ref{thrm_intro_disp_2}) does not need this requirement.
Similarly, we use the clipping function $\clip$ to ensure the validity of (\ref{thrm_intro_disp_1}), which in \cite{JenWelti2020} is formulated for networks with clipping function as read-out map.

Our arguments are based on an analysis of regions in the parameter space related to ``inactive" neurons.
In these regions, the realization function is constant not only in its argument but also in the network parameter.
For example, if $\NNel$ contains only strictly negative parameters, then $\ReLU \circ \NNaff{\arch_2}{\arch_1}{\NNel,\arch_1(\arch_0+1)} \circ \ReLU(x)$ is constantly zero in $x$ and in a neighborhood of $\NNel$.
As a consequence, SGD will not be able to escape from such a region.
The fact that random initialization can render parts of a ReLU network inactive has already been noticed in
\cite{LuShinSuKarn2019,ShinKarn2020}.
While the focus of \cite{LuShinSuKarn2019,ShinKarn2020}
is on the design of alternative random initialization schemes to make the training more efficient, we here give precise estimates on the probability that the whole network becomes inactive and deduce that SGD fails to converge if the number of random initializations does not increase fast enough.
Note that in \eqref{thrm_intro_disp_2} we take the limit superior over all architectures $(\arch_0,\dots,\arch_{\depth}) \in A_d$ whose depth $\depth$ and minimal width $\min\{\arch_1,\dots,\arch_{\depth-1}\}$ both tend to infinity.
In particular, to prove \eqref{thrm_intro_disp_2}, it is sufficient to construct a single sequence of such architectures over which the limit is positive.
For the sequence we use, the depth grows much faster than the maximal width $\max\{\arch_1,\dots,\arch_{\depth-1}\}$.
This imbalance between depth and width has the effect that the training procedure does not converge.

The remainder of this article is organized as follows.
In \cref{section_algorithm}, we provide an abstract version of the SGD algorithm for training neural networks in a supervised learning framework.
\cref{section_inactive} contains preliminary results on inactive neurons and constant network realization functions.
In \cref{section_quant}, we discuss the consequences of these preliminary results for the convergence of the SGD method, and \cref{section_qual} contains our main results, \cref{thrm_main_quant} and \cref{cor_main_qual}.


\section{Mathematical description of the SGD method}
\label{section_algorithm}

In this section, we give a mathematical description of an abstract version of the SGD algorithm for training neural networks in a supervised learning framework.
To do that, we slightly generalize the setup of the introduction.
We begin with an informal description and give a precise formulation afterwards.
First, fix a network architecture $\arch = (a_0, \dots, a_D) \in \NNset$.
Let $\mathcal{X} \colon \Psp \rightarrow [u,v]^d$ and $\Bayes \colon \Psp \rightarrow [\mathfrak{u},\mathfrak{v}]$ be random variables on a probability space $(\Omega, {\cal F}, \mathbb{P})$, on which the true risk $\trueLoss{}(\NNel) = \E[ |(\clip \circ \NNfct{\arch}{\NNel})(\mathcal{X}) - \Bayes| ]$ of a network $\NNel \in \R^{\NNparam}$ is based.
Here, $\clip \colon \R \to \R$ can be any continuous function, which covers the case of network realizations 
with non-linear read-out maps.
In the context of the introduction, $\Bayes$ stands for the random variable $\expB(\datX)$.
Throughout, $n \in \N$ will denote the index of the gradient trajectory and $t \in \N_0$ the index of the gradient step.
$\fctLoss{n}{t}{}$ denotes the empirical risk defined on the space of functions $C(\R^d,\R)$.
In this general setting, $\fctLoss{n}{t}{}$ can be any function from $C(\R^d,\R) \times \Psp$ to $\R$, but the specific example we have in mind is
\begin{equation*}
	\fctLoss{n}{t}{}(f) = \frac{1}{m} \ssum{j=1}{m} \big| f(\dataX{n}{t}{j}) - \dataY{n}{t}{j} \big|^2
\end{equation*}%
for a given batch size $m \in \N$.
$\Loss{n}{t}{}$ is the empirical risk defined on the space of network parameters, given in terms of $\fctLoss{n}{t}{}$ by $\Loss{n}{t}{}(\NNel) = \fctLoss{n}{t}{}(\clip \circ \NNfct{\arch}{\NNel})$.
Let $\SGDgrad{n}{t}{} \colon \R^{\NNparam} \times \Psp \rightarrow \R^{\NNparam}$ be a function that agrees with the gradient of $\Loss{n}{t}{}$ where it exists.
Then, we can introduce the gradient trajectories $\SGDstep{n}{t}{} \colon \Psp \rightarrow \R^{\NNparam}$ satisfying
\begin{equation*}
	\SGDstep{n}{t}{} = \SGDstep{n}{t-1}{} - \SGDstepsize_t \SGDgrad{n}{t}{}(\SGDstep{n}{t-1}{})
\end{equation*}%
for given step sizes $\gamma_t$.
The $N$ random initializations $\SGDstep{n}{0}{}$, $n \in \{1,\dots,N\}$, are assumed to be i.i.d in $n$ and have independent marginals.
Lastly, $\SGDmin{} \colon \Psp \rightarrow \N \times \N_0$ specifies the output of the algorithm consisting of a pair of indices for a gradient trajectory and a gradient step. The expected true risk is $\finalLoss = \E[\min\{\trueLoss{}(\SGDvstep{\SGDmin{}}{}),1\}]$. In the following, we present the formal algorithm.

\begin{setting}
\label{SGD_setup}
Let $u,\mathfrak{u} \in \R$, $v \in (u,\infty)$, $\mathfrak{v} \in (\mathfrak{u},\infty)$, $\clip \in C(\R,\R)$, $d,\depth,N \in \N$, $\arch = (\arch_0,\dots,\arch_{\depth}) \in \NNset$, and $(\SGDstepsize_t)_{t \in \N} \subseteq \R$. Consider random variables 
$\mathcal{X} \colon \Psp \rightarrow [u,v]^d$ and $\Bayes \colon \Psp \rightarrow [\mathfrak{u},\mathfrak{v}]$
on a probability space $(\Psp,\Psa,\Pm)$. Let $\trueLoss{} \colon \R^{\NNparam} \rightarrow [0,\infty]$ be given by $\trueLoss{}(\NNel) = \E[ |(\clip \circ \NNfct{\arch}{\NNel})(\mathcal{X}) - \Bayes| ]$. For all $n \in \N$ and $t \in \N_0$,
let $\fctLoss{n}{t}{}$ be a function from $C(\R^d,\R) \times \Psp$ to $\R$, and denote by 
$\Loss{n}{t}{} \colon \R^{\NNparam} \times \Psp \rightarrow \R$ the mapping given by
$\Loss{n}{t}{}(\NNel) = \fctLoss{n}{t}{}(\clip \circ \NNfct{\arch}{\NNel})$. Let 
$\SGDgrad{n}{t}{} = (\SGDgrad{n}{t}{1},\dots,\SGDgrad{n}{t}{\NNparam}) \colon \R^{\NNparam} \times \Psp \rightarrow \R^{\NNparam}$ be a function satisfying 
\begin{equation}
\label{SGD_gradient}
	\SGDgrad{n}{t}{i}(\NNel,\Pel) = \frac{\partial}{\partial \NNel_i} \Loss{n}{t}{}(\NNel,\Pel)
\end{equation}%
for all  $n,t \in \N$, $i \in \{1,\dots,\NNparam\}$, $\Pel \in \Psp$, and 
\begin{equation*}
	\NNel = (\NNel_1,\dots,\NNel_{\NNparam}) \in \left\{ \vartheta = (\vartheta_1,\dots,\vartheta_{\NNparam}) \in \R^{\NNparam}\colon
	\!\begin{array}{cc}
		\Loss{n}{t}{}(\vartheta_1,\dots,\vartheta_{i-1},(\cdot), \vartheta_{i+1},\dots,\vartheta_{\NNparam},\Pel) \\
		\text{ as a function } \R \rightarrow \R \text{ is differentiable at } \vartheta_i.
	\end{array} \right\}.
\end{equation*}%
Let $\SGDstep{n}{t}{} = (\SGDstep{n}{t}{1},\dots,\SGDstep{n}{t}{\NNparam}) \colon \Psp \rightarrow \R^{\NNparam}$, $n \in \N$, $t \in \N_0$, be random variables such that $\SGDstep{1}{0}{},\dots,\SGDstep{N}{0}{}$ are i.i.d., $\SGDstep{1}{0}{1},\dots,\SGDstep{1}{0}{\NNparam}$ are independent, and 
\begin{equation}
\label{SGD_step}
	\SGDstep{n}{t}{} = \SGDstep{n}{t-1}{} - \SGDstepsize_t \SGDgrad{n}{t}{}(\SGDstep{n}{t-1}{})
\end{equation}%
for all $n,t \in \N$. Let $\SGDmin{} \colon \Psp \rightarrow \{1,\dots,N\} \times \N_0$ be a random variable,
and denote $\finalLoss = \E[\min\{\trueLoss{}(\SGDvstep{\SGDmin{}}{}),1\}]$.
\end{setting}

Note that, by \cite[Lemma 6.2]{JenWelti2020} and Tonelli's theorem, it follows from \cref{SGD_setup} that $\trueLoss{}(\SGDvstep{\SGDmin{}}{}) \colon \Psp \rightarrow [0,\infty]$ is measurable and, as a consequence, $\finalLoss = \E[\min\{\trueLoss{}(\SGDvstep{\SGDmin{}}{}),1\}]$ is well-defined.


\section{DNNs with constant realization functions}
\label{section_inactive}

In this section, we study a subset of the parameter space, specified in \cref{def_inactive_set} below, for which neurons in a DNN become ``inactive'', rendering the realization function of the DNN constant.
We deduce a few properties for such DNNs in
\cref{lem_inactive_NN_is_const,lem_inactive_large_error,lem_SGD_stays_inactive} below.
The material in this section is related to the findings in
\cite{LuShinSuKarn2019,ShinKarn2020}.

\begin{definition}
\label{def_inactive_set}
Let $\depth \in \N$ and $\arch = (\arch_0,\dots,\arch_{\depth}) \in \N^{\depth+1}$.
For all $j \in \N \cap (0,\depth)$, let $\InActL{j} \subseteq \R^{\NNparam}$ be the set
\begin{equation*}
	\InActL{j} = \bigg\{ \NNel = (\NNel_1,\dots,\NNel_{\NNparam}) \in \R^{\NNparam} \colon \bigg[ \forall\, k \in \N \cap \bigg(\ssum{i=1}{j-1} \arch_i(\arch_{i-1}+1),\ssum{i=1}{j} \arch_i(\arch_{i-1}+1)\bigg] \colon \NNel_k < 0 \bigg] \bigg\},
\end{equation*}%
and denote $\InAct = \bigcup_{j \in \N \cap (1,\depth)} \InActL{j}$.
\end{definition}

First, we verify that the realization function is constant in both the argument and the network parameter on certain subsets of $\InActL{j}$.

\begin{lemma}
\label{lem_inactive_NN_is_const}
	Let $\depth \in \N$, $j \in \N \cap (1,\depth)$, $\arch = (\arch_0,\dots,\arch_{\depth}) \in \N^{\depth+1}$, $\NNel = (\NNel_1,\dots,\NNel_{\NNparam})$, $\vartheta = (\vartheta_1,\dots,\vartheta_{\NNparam}) \in \InActL{j}$, $x \in \R^{\arch_0}$, and assume that $\NNel_k = \vartheta_k$ for all $k \in \N \cap \big(\ssum{i=1}{j} \arch_i(\arch_{i-1}+1),\NNparam\big]$.
Then $\NNfct{\arch}{\NNel}(0) = \NNfct{\arch}{\NNel}(x) = \NNfct{\arch}{\vartheta}(x) = \NNfct{\arch}{\vartheta}(0)$.
\end{lemma}

\begin{proof}
	For all $k \in \{1,\dots,\depth\}$, denote $m_k = \ssum{i=1}{k} \arch_i(\arch_{i-1}+1)$.
Since, by assumption, $\NNel,\vartheta \in \InActL{j}$, one has for all $k \in \N \cap (m_{j-1},m_j]$ that $\NNel_k < 0$ and $\vartheta_k < 0$.
This and $\ReLU(\R^{\arch_{j-1}}) = [0,\infty)^{\arch_{j-1}}$ imply for all $y \in \R^{\arch_{j-1}}$, $\phi \in \{\NNel,\vartheta\}$ that $\NNaff{\arch_j}{\arch_{j-1}}{\phi,m_{j-1}} \circ \ReLU(y) \in (-\infty,0]^{\arch_j}$.
This ensures for all $y \in \R^{\arch_{j-1}}$, $\phi \in \{\NNel,\vartheta\}$ that $\rho \circ \NNaff{\arch_j}{\arch_{j-1}}{\phi,m_{j-1}} \circ \ReLU(y) = 0$.
Moreover, the assumption that $\NNel_k = \vartheta_k$ for all
$k \in \N \cap \big(\ssum{i=1}{j} \arch_i(\arch_{i-1}+1),\NNparam\big]$ yields $\NNaff{\arch_k}{\arch_{k-1}}{\NNel,m_{k-1}} = \NNaff{\arch_k}{\arch_{k-1}}{\vartheta,m_{k-1}}$
for all $k \in \N \cap (j,\depth]$. This implies that $\NNfct{\arch}{\NNel}(y) = \NNfct{\arch}{\vartheta}(z)$
for all $y,z \in \R^{\arch_{0}}$, which completes the proof of \cref{lem_inactive_NN_is_const}.
\end{proof}

The next lemma shows that networks with parameters in $\InAct$ cannot perform better than a constant solution to the learning task.

\begin{lemma}
\label{lem_inactive_large_error}
	Assume \cref{SGD_setup} and let $\NNel \in \InAct$.
Then $\trueLoss{}(\NNel) \geq \inf_{b \in \R} \E[|b-\Bayes|]$.
\end{lemma}

\begin{proof}
Let $\zeta \in \Psp$. By \cref{lem_inactive_NN_is_const}, one has $\NNfct{\arch}{\NNel}(x) = \NNfct{\arch}{\NNel}(0)$
for all $x \in \R^d$. Therefore, we obtain $\NNfct{\arch}{\NNel}(\mathcal{X}(\Pel)) = \NNfct{\arch}{\NNel}(\mathcal{X}(\zeta))$ for all $\Pel \in \Psp$. In particular, $\trueLoss{}(\NNel) = \E[ |(\clip \circ \NNfct{\arch}{\NNel})(\mathcal{X}(\zeta)) - \Bayes| ] \geq \inf_{b \in \R} \E[|b-\Bayes|]$.
\end{proof}

Finally, we show that SGD cannot escape from $\InAct$.

\begin{lemma}
\label{lem_SGD_stays_inactive}
Assume \cref{SGD_setup} and let $n,t \in \N$, $\Pel \in \Psp$, $j \in \N \cap (1,\depth)$. Suppose that $\SGDstep{n}{0}{}(\Pel) \in \InActL{j}$.
Then $\SGDstep{n}{t}{}(\Pel) \in \InActL{j}$.
\end{lemma}

\begin{proof}
	Denote $m_0 = \ssum{i=1}{j-1} \arch_i(\arch_{i-1}+1)$ and $m_1 = \ssum{i=1}{j} \arch_i(\arch_{i-1}+1)$.
We prove by induction that for all $s \in \N_0$ we have $\SGDstep{n}{s}{}(\Pel) \in \InActL{j}$.
The case $s=0$ is true by assumption.
Now suppose that $s \in \N_0$ and $\NNel = (\NNel_1,\dots,\NNel_{\NNparam}) \in \R^{\NNparam}$ satisfy $\NNel = \SGDstep{n}{s}{}(\Pel) \in \InActL{j}$.
Let $U \subseteq \R^{\NNparam}$ be the set given by $U = \{(\NNel_1,\dots,\NNel_{m_0})\} \times (-\infty,0)^{m_1-m_0} \times \{(\NNel_{m_1+1},\dots,\NNel_{\NNparam})\}$.
Then $\NNel \in U \subseteq \InActL{j}$.
By \cref{lem_inactive_NN_is_const}, we have $\NNfct{\arch}{\phi}(x) = \NNfct{\arch}{\NNel}(x)$
for all $\phi \in U$ and $x \in \R^{d}$. Hence, $\Loss{n}{s+1}{}(\phi,\Pel) = \Loss{n}{s+1}{}(\NNel,\Pel)$ for all $\phi \in U$
and, as a consequence, $\frac{\partial}{\partial \NNel_k} \Loss{n}{s+1}{}(\NNel,\Pel) = 0$
for all $k \in \N \cap (m_0, m_1]$.
So, it follows from (\ref{SGD_gradient}), (\ref{SGD_step}), and the induction hypothesis that $\SGDstep{n}{s+1}{}(\Pel) \in \InActL{j}$, which completes the proof of \cref{lem_SGD_stays_inactive}.
\end{proof}


\section{Quantitative lower bounds for the SGD method in the training of DNNs}
\label{section_quant}

In this section, we establish in \cref{prop_quant} below a quantitative lower bound for the error of the SGD method in the training of DNNs.

\begin{lemma}
\label{lem_prob_bound}
	Assume \cref{SGD_setup} and suppose $\depth \geq 3$. For all $j \in \{1,\dots,\depth-1\}$, denote $k_j = \ssum{i=1}{j} \arch_i(\arch_{i-1}+1)$, $p = \inf_{i \in \{1,\dots,\NNparam\}} \Pm(\SGDstep{1}{0}{i} < 0)$, and $\maxwidth = \max\{\arch_1,\dots,\arch_{\depth-1}\}$.
Then
\begin{equation*}
\begin{split}
	\Pm\big(\forall\, n \in \{1,\dots,N\}, \, t \in \N_0 \colon \SGDstep{n}{t}{} \in \InAct\big) &= \bigg[ 1 - \sprod{j=2}{\depth-1} \Big( 1 - \sprod{i=1+k_{j-1}}{k_j} \Pm( \SGDstep{1}{0}{i} < 0 ) \Big) \bigg]^N \\
	&\geq \big[1-(1-p^{\maxwidth(\maxwidth+1)})^{\depth-2}\big]^N.
\end{split}
\end{equation*}%
\end{lemma}

\begin{proof}
It follows from the independence of $\SGDstep{1}{0}{1},\dots,\SGDstep{1}{0}{\NNparam}$ that
\begin{equation*}
	\Pm(\SGDstep{1}{0}{} \in \InAct) = \Pm\big( \exists\, j \in \N \cap (1,\depth) \colon \forall\, i \in \N \cap (k_{j-1},k_j] \colon \SGDstep{1}{0}{i} < 0 \big) = 1 - \sprod{j=2}{\depth-1} \Big( 1 - \sprod{i=1+k_{j-1}}{k_j} \Pm( \SGDstep{1}{0}{i} < 0 ) \Big).
\end{equation*}%
By definition of $p$ and $W$, the right hand side is greater than or equal to $1-(1-p^{\maxwidth(\maxwidth+1)})^{\depth-2}$.
Moreover, \cref{lem_SGD_stays_inactive} and the assumption that $\SGDstep{1}{0}{},\dots,\SGDstep{N}{0}{}$ are i.i.d.\ yield
\begin{equation*}
\begin{split}
	\Pm\big(\forall\, n \in \{1,\dots,N\}, \, t \in \N_0 \colon \SGDstep{n}{t}{} \in \InAct\big) &= \Pm\big(\forall\, n \in \{1,\dots,N\} \colon \SGDstep{n}{0}{} \in \InAct\big) = \big( \Pm(\SGDstep{1}{0}{} \in \InAct) \big)^N,
\end{split}
\end{equation*}%
which completes the proof of \cref{lem_prob_bound}.
\end{proof}

\begin{proposition}
\label{prop_quant}
	Under the same assumptions as in \cref{lem_prob_bound}, one has
\begin{equation}
\label{prop_quant_disp}
\begin{split}
	\finalLoss = \E[\min\{\trueLoss{}(\SGDvstep{\SGDmin{}}{}),1\}] &\geq \bigg[ 1 - \sprod{j=2}{\depth-1} \Big( 1 - \sprod{i=1+k_{j-1}}{k_j} \Pm( \SGDstep{1}{0}{i} < 0 ) \Big) \bigg]^N \min\!\Big\{\inf_{b \in \R} \E[|b-\Bayes|],1\Big\} \\
	&\geq \big[1-(1-p^{\maxwidth(\maxwidth+1)})^{\depth-2}\big]^N \min\!\Big\{\inf_{b \in \R} \E[|b-\Bayes|],1\Big\}.
\end{split}
\end{equation}%
\end{proposition}

\begin{proof}
Denote $C = \min\{\inf_{b \in \R} \E[|b-\Bayes|],1\}$ and observe that 
\cref{lem_inactive_large_error} implies for all $\Pel \in \Psp$ with $\SGDvstep{\SGDmin{}(\Pel)}{}(\Pel) \in \InAct$ that $\min\{\trueLoss{}(\SGDvstep{\SGDmin{}(\Pel)}{}(\Pel)),1\} \geq C$.
Markov's inequality hence ensures that
\begin{equation*}
	C \, \Pm( \SGDvstep{\SGDmin{}}{} \in \InAct ) \leq C \, \Pm(\min\{\trueLoss{}(\SGDvstep{\SGDmin{}}{}),1\} \geq C) \leq \finalLoss.
\end{equation*}%
Combining this with \cref{lem_prob_bound} and the fact that $\Pm( \SGDvstep{\SGDmin{}}{} \in \InAct ) \geq \Pm(\forall\, n \in \{1,\dots,N\}, \, t \in \N_0 \colon \SGDstep{n}{t}{} \in \InAct)$ establishes (\ref{prop_quant_disp}).
\end{proof}

Let us briefly discuss how the inequality in \cref{prop_quant} relates to prior work in the literature.
Fix a depth $\depth \in \N$ and consider the problem of distributing a given number of neurons among the $\depth-1$ hidden layers. In order to minimize the chance of starting with an inactive network, one needs to minimize the quantity $1 - \sprod{j=2}{\depth-1} ( 1 - \sprod{i=1+k_{j-1}}{k_j} \Pm( \SGDstep{1}{0}{i} < 0 ) )$ from \eqref{prop_quant_disp}.
Under the assumption that $\Pm( \SGDstep{1}{0}{i} < 0 )$ does not depend on $i$, this can be achieved by choosing 
the same number of neurons in each layer.

The effects of initialization and architecture on early training have also been studied in \cite{Hanin2018,HanRol2018}.
While \cite{Hanin2018} investigates the problem of vanishing and exploding gradients, \cite{HanRol2018} studies two failure modes associated with poor starting conditions. Both find that, given a total number of neurons to spend, distributing 
them evenly among the hidden layers, yields the best results. This is in line with our findings.


\section{Main results}
\label{section_qual}

In this section, we prove the paper's main results, \cref{thrm_main_quant} and \cref{cor_main_qual}.
While \cref{thrm_main_quant} provides precise quantitative conditions under which SGD does not converge in the training of DNNs, \cref{cor_main_qual} is a qualitative result. To prove them, we need the following elementary result.
Throughout, $\log$ denotes the natural logarithm.

\begin{lemma}
\label{lem_kappa_bound}
Let $\depth,N,\maxwidth \in (0,\infty)$ and $\kappa,p \in (0,1)$ be such that
$\depth \geq |\!\log(p)|\maxwidth p^{-\maxwidth}$ and $N \leq |\!\log(\kappa)|(1-p^{\maxwidth})^{1-\depth}$.
Then $[1 - (1-p^{\maxwidth})^{\depth}]^N \geq \kappa$.
\end{lemma}

\begin{proof}
Let the functions $f \colon [0,1) \rightarrow \R$ and $g \colon [0,1) \rightarrow \R$ be given by $f(x) = x+\log(1-x)$ and $g(x) = (1-p^{\maxwidth})^{-1}x+\log(1-x)$. Since $f(0) = 0$ and $f'(x) = 1-(1-x)^{-1} < 0$ for all $x \in (0,1)$, one has
$|\!\log(1-x)|^{-1} < x^{-1}$ for all $x \in (0,1)$.
Hence, $\depth > |\!\log(p)| \maxwidth |\!\log(1-p^{\maxwidth})|^{-1}$, from which it follows that 
$(1-p^{\maxwidth})^\depth < p^{\maxwidth}$. In addition, $g(0) = 0$ and
$g'(x) = (1-p^{\maxwidth})^{-1} - (1-x)^{-1} > 0$ for all $x \in (0,p^{\maxwidth})$, which implies that
$|\!\log(1-x)| < (1-p^{\maxwidth})^{-1}x$ for all $x \in (0,p^{\maxwidth})$. Hence, we deduce from 
$(1-p^{\maxwidth})^{\depth} < p^{\maxwidth}$ that 
$N |\!\log(1-(1-p^{\maxwidth})^{\depth})| < N (1-p^{\maxwidth})^{\depth-1} \leq |\!\log(\kappa)|$, and 
taking the exponential yields the desired statement.
\end{proof}

We proved \cref{prop_quant} in the abstract framework of \cref{SGD_setup}.
For the sake of concreteness, we now return to the setup of the introduction.
We quickly recall it below.

\begin{setting}
\label{SGD_setup_general}
	Let $u,\mathfrak{u} \in \R$, $v \in (u,\infty)$, $\mathfrak{v} \in (\mathfrak{u},\infty)$, $\clip \in C(\R,\R)$, $d \in \N$, and $(\SGDstepsize_t)_{t \in \N} \subseteq \R$.
Consider functions $\dataX{n}{t}{j} \colon \Psp \rightarrow [u,v]^d$ and $\dataY{n}{t}{j} \colon \Psp \rightarrow [\mathfrak{u},\mathfrak{v}]$, $j,n,t \in \N_0$, on a probability space $(\Psp,\Psa,\Pm)$ such that $\datX$ and $\datY$ are random variables.
Let $\expB \colon [u,v]^d \rightarrow [\mathfrak{u},\mathfrak{v}]$ be a measurable function such that $\Pm$-a.s.\ $\expB(\datX) = \E[\datY | \datX]$. Let
$\Loss{n}{t}{\arch,m} \colon \R^{\NNparam} \times \Psp \rightarrow \R$, $m \in \N$, $n,t \in \N_0$, $\arch \in \NNset$, be given by
\begin{equation}
\label{SGD_setup_gen_emp_loss}
\Loss{n}{t}{\arch,m}(\NNel) = \frac{1}{m} \ssum{j=1}{m} \big| (\clip \circ \NNfct{\arch}{\NNel})(\dataX{n}{t}{j}) - \dataY{n}{t}{j} \big|^2,
\end{equation}%
and assume $\SGDgrad{n}{t}{\arch,m} = (\SGDgrad{n}{t}{\arch,m,1},\dots,\SGDgrad{n}{t}{\arch,m,\NNparam}) \colon \R^{\NNparam} \times \Psp \rightarrow \R^{\NNparam}$ are mappings satisfying
\begin{equation*}
	\SGDgrad{n}{t}{\arch,m,i}(\NNel,\Pel) = \frac{\partial}{\partial \NNel_i} \Loss{n}{t}{\arch,m}(\NNel,\Pel)
\end{equation*}%
for all $m,n,t \in \N$, $\arch \in \NNset$, $i \in \{1,\dots,\NNparam\}$, $\Pel \in \Psp$, and
\begin{equation*}
	\NNel = (\NNel_1,\dots,\NNel_{\NNparam}) \in \left\{ \vartheta = (\vartheta_1,\dots,\vartheta_{\NNparam}) \in \R^{\NNparam} \colon
	\!\begin{array}{cc}
		\Loss{n}{t}{\arch,m}(\vartheta_1,\dots,\vartheta_{i-1},(\cdot), \vartheta_{i+1},\dots,\vartheta_{\NNparam},\Pel) \\
		\text{ as a function } \R \rightarrow \R \text{ is differentiable at } \vartheta_i.
	\end{array} \right\}.
\end{equation*}%
Let $\SGDstep{n}{t}{\arch,m} = (\SGDstepM{n}{t}{\arch,m}{1},\dots,\SGDstepM{n}{t}{\arch,m}{\NNparam}) \colon \Psp \rightarrow \R^{\NNparam}$, $m,n \in \N$, $t \in \N_0$, $\arch \in \NNset$, be random variables such that $\SGDstep{n}{0}{\arch,m}$, $n \in \N$, are i.i.d., $\SGDstepM{1}{0}{\arch,m}{1},\dots,\SGDstepM{1}{0}{\arch,m}{\NNparam}$ are independent for all $m \in \N$, $\arch \in \NNset$, and
\begin{equation*}
	\SGDstep{n}{t}{\arch,m} = \SGDstep{n}{t-1}{\arch,m} - \SGDstepsize_t \SGDgrad{n}{t}{\arch,m}(\SGDstep{n}{t-1}{\arch,m})
\end{equation*}%
for all $m,n,t \in \N$, $\arch \in \NNset$. Let $\SGDmin{\arch,M,N,T} \colon \Psp \rightarrow \{1,\dots,N\} \times \N_0$, $M,N \in \N$, $T \in \N_0$, $\arch \in \NNset$, be random variables.
\end{setting}

The following is the main result of this article.

\begin{theorem}
\label{thrm_main_quant}
Assume \cref{SGD_setup_general} and fix $M \in \N$. Consider sequences
$(\depth_l,N_l,\maxwidth_l)_{l \in \N_0} \subseteq \N^3$, 
$(\arch^l)_{l \in \N_0} = (\arch_0^l,\dots,\arch_{\depth_l}^l)_{l \in \N_0} \subseteq \NNset$ and
constants $\kappa,p \in (0,1)$ such that, for all $l \in \N_0$,
$\maxwidth_l = \max\{\arch_1^l,\dots,\arch_{\depth_l-1}^l\}$, 
$\depth_l \geq |\!\log(p)| \maxwidth_l (\maxwidth_l+1) p^{-\maxwidth_l(\maxwidth_l+1)} + 2$,
and $N_l \leq |\!\log(\kappa)|(1-p^{\maxwidth_l(\maxwidth_l+1)})^{3-\depth_l}$. Let $\Phi_{l,T} \colon \Psp \rightarrow \R^{\NNparaM{l}}$, $l,T \in \N_0$, be given by
$\Phi_{l,T} = \SGDvstep{\SGDmin{\arch^l,M,N_l,T}}{\arch^l,M}$,
and assume that $\inf_{l \in \N_0} \inf_{i \in \{1,\dots,\NNparaM{l}\}} \Pm \big( \SGDstepM{1}{0}{\arch^l,M}{i} < 0 \big) \geq p$.
Then
\begin{equation}
\label{thrm_main_quant_conclusion}
\liminf_{l \rightarrow \infty} \inf_{T \in \N_0} \E\bigg[\min\!\bigg\{ \int_{[u,v]^d} \abs{ \brak{ \clip \circ \NNfct{\arch^l}{\Phi_{l,T}}} (x) - \expB(x) }\, \Pm_{\datX}(dx) , 1 \bigg\}\bigg] \geq \kappa \min\!\Big\{\inf_{b \in \R} \E\big[|b-\expB(\datX)|\big],1\Big\}.
\end{equation}%
\end{theorem}

\begin{proof}
	Denote $q = \inf_{l \in \N_0} \inf_{i \in \{1,\dots,\NNparaM{l}\}} \Pm \big( \SGDstepM{1}{0}{\arch^l,M}{i} < 0 \big)$.
By \cref{prop_quant}, one has, for all $l,T \in \N_0$,
\begin{equation*}
\begin{split}
&\E \bigg[\min \! \bigg\{ \int_{[u,v]^d} \abs{ \brak{ \clip \circ \NNfct{\arch^l}{\Phi_{l,T}}} (x) - \expB(x) }\, \Pm_{\datX}(dx) , 1 \bigg\}\bigg] \\
&\geq \big[1-(1-q^{\maxwidth_l(\maxwidth_l+1)})^{\depth_l-2}\big]^{N_l} \min\!\Big\{\inf_{b \in \R} \E\big[|b-\expB(\datX)|\big],1\Big\}.
\end{split}
\end{equation*}%
Moreover, \cref{lem_kappa_bound} implies that, for all $l \in \N_0$,
\begin{equation*}
	\big[1-(1-q^{\maxwidth_l(\maxwidth_l+1)})^{\depth_l-2}\big]^{N_l} \geq \big[1-(1-p^{\maxwidth_l(\maxwidth_l+1)})^{\depth_l-2}\big]^{N_l} \geq \kappa,
\end{equation*}%
which completes the proof of \cref{thrm_main_quant}.
\end{proof}

Instead of focusing on a single sequence of architectures as in \cref{thrm_main_quant}, one can instead consider the limit superior over all possible architectures, which we do in \cref{cor_main_qual} below.
Note that this allows us to increase the constant $\kappa$ from \eqref{thrm_main_quant_conclusion} to 1.

\begin{corollary}
\label{cor_main_qual}
Assume \cref{SGD_setup_general} and let $c \in (0,\infty)$. Suppose that
$\mathrm{Var}(\expB(\datX)) > 0$ and assume that $\SGDstep{n}{0}{\arch,m}$ is uniformly distributed on $[-c,c]^{\NNparam}$
for all $m,n \in \N$, $\arch \in \NNset$. Then
\begin{equation}
\label{cor_main_qual_conclusion}
\begin{split}
\inf_{N \in \N} \limsup_{\substack{\arch = (\arch_0,\dots,\arch_{\depth}) \in \NNset \\ \min\{\depth,\arch_1,\dots,\arch_{\depth-1}\} \rightarrow \infty}} \inf_{\substack{M \in \N \\ T \in \N_0}} \E\bigg[\min\!\bigg\{ \int_{[u,v]^d} \abs{\brak{ \clip \circ \NNfct{\arch}{\SGDvstep{\SGDmin{\arch,M,N,T}}{\arch,M}} } (x) - \expB(x)}\, \Pm_{\datX}(dx) , 1 \bigg\}\bigg] \\
\geq \min\!\Big\{\inf_{b \in \R} \E\big[|b-\expB(\datX)|\big],1\Big\} > 0.
\end{split}
\end{equation}%
\end{corollary}

\begin{proof}
	First note that $\inf_{M \in \N} \inf_{\arch \in \NNset} \inf_{i \in \{1,\dots,\NNparam\}} \Pm(\SGDstepM{1}{0}{\arch,M}{i} < 0) = \frac{1}{2}$.
So, it follows from \cref{thrm_main_quant} that for all $k,N \in \N$, $\kappa \in (0,1)$ there exist $\depth \in \N$ and $\arch = (\arch_0,\dots,\arch_{\depth}) \in \NNset$ such that $\min\{\depth,\arch_1,\dots,\arch_{\depth-1}\} \geq k$ and
\begin{equation*}
\label{cor_main_qual_PF}
	\inf_{T \in \N_0} \E\bigg[\min\!\bigg\{ \int_{[u,v]^d} \big| \big( \clip \circ \NNfct{\arch}{\SGDvstep{\SGDmin{\arch,M,N,T}}{\arch,M}} \big) (x) - \expB(x) \big|\, \Pm_{\datX}(dx) , 1 \bigg\}\bigg] \geq \kappa \min\!\Big\{\inf_{b \in \R} \E\big[|b-\expB(\datX)|\big],1\Big\}
\end{equation*}%
for all $M \in \N$. As a result, one has
\begin{equation*}
\begin{split}
	\inf_{N \in \N} \limsup_{\substack{\arch = (\arch_0,\dots,\arch_{\depth}) \in \NNset \\ \min\{\depth,\arch_1,\dots,\arch_{\depth-1}\} \rightarrow \infty}} \inf_{\substack{M \in \N \\ T \in \N_0}} \E\bigg[\min\! \bigg\{ \int_{[u,v]^d} \abs{ \brak{ \clip \circ \NNfct{\arch}{\SGDvstep{\SGDmin{\arch,M,N,T}}{\arch,M}} } (x) - \expB(x) }\, \Pm_{\datX}(dx) , 1 \bigg\}\bigg] \\
	\geq \kappa \min\!\Big\{\inf_{b \in \R} \E\big[|b-\expB(\datX)|\big],1\Big\}
\end{split}
\end{equation*}%
for all $\kappa \in (0,1)$. Taking the limit $\kappa \uparrow 1$ and noting that the assumption $\mathrm{Var}(\expB(\datX)) > 0$ implies $\inf_{b \in \R} \E\big[|b-\expB(\datX)|\big] > 0$ completes the proof of the corollary.
\end{proof}


\vskip 5mm\noindent{\large\sc Acknowledgments}\vskip 2mm
\noindent This work has partially been supported by Swiss National Science Foundation Research Grant 175699.
The second author acknowledges funding by the Deutsche Forschungsgemeinschaft (DFG, German Research Foundation) under Germany's Excellence Strategy EXC 2044-390685587, Mathematics Muenster: Dynamics-Geometry-Structure.

\phantomsection%
\addcontentsline{toc}{section}{Bibliography}%
\bibliographystyle{acm}%
\bibliography{}

\end{document}